\documentclass[letterpaper, 10 pt, conference]{ieeeconf}  

\IEEEoverridecommandlockouts                              

\usepackage[usenames, dvipsnames]{color}
\usepackage{listings}
\usepackage{xcolor}

\usepackage{float}
\usepackage{amsmath,amssymb,amsthm}
\usepackage{mathtools} 
\usepackage[english]{babel}

\usepackage{epsfig}
\usepackage{epstopdf}
\usepackage{subcaption}
\usepackage{multirow}
\usepackage{graphicx}

\usepackage[export]{adjustbox}
\usepackage[boxruled,vlined,linesnumbered]{algorithm2e}

\usepackage{hyperref}
\usepackage [autostyle, english = american]{csquotes}
\MakeOuterQuote{"}

\newtheorem{theorem}{Theorem}

\newtheorem{corollary}[theorem]{Corollary}

\title{\LARGE \bf
	Relevant Region Exploration On General Cost-maps For Sampling-Based Motion Planning
}

\author{ Sagar Suhas Joshi$^{1}$ and Panagiotis Tsiotras$^{2}$ 
\thanks{$^{1}$Robotics PhD student at Georgia Institute of Technology, USA. sagarsjoshi94@gmail.com}
\thanks{$^{2}$Professor and David and Andrew Lewis Chair in the Guggenheim School of Aerospace Engineering at Georgia Institute of Technology, USA. }
\thanks{This work has been supported by NSF award IIS-1617630.}
\thanks{This work was presented at IEEE/RSJ International Conference on Intelligent Robots and Systems (IROS), Las Vegas, NV, 2020.}
\thanks{DOI: 10.1109/IROS45743.2020.9340806}
}

\renewcommand{\baselinestretch}{0.98}

\begin{document}
	
	\maketitle
	\thispagestyle{empty}
	\pagestyle{empty}
	
	\begin{abstract}
		Asymptotically optimal sampling-based planners require an intelligent exploration strategy to accelerate convergence. 
		After an initial solution is found, a necessary condition for improvement is to generate new samples in the so-called "Informed Set". 
		However, Informed Sampling can be ineffective in focusing search if the chosen heuristic fails to provide a good estimate of the solution cost. 
		This work proposes an algorithm to sample the "Relevant Region" instead, which is a subset of the Informed Set. 
		The Relevant Region utilizes cost-to-come information from the planner's tree structure, reduces dependence on the heuristic, and further focuses the search.
		Benchmarking tests in uniform and general cost-space settings demonstrate the efficacy of Relevant Region sampling.    
	\end{abstract}
	\section{INTRODUCTION}
	In recent years, sampling-based motion planning (SBMP) algorithms have gained popularity due to their ability to handle high dimensional search spaces.
	Deterministic search methods such as A* do not scale well owing to the computational cost associated with the a priori discretization of the search space.
	Incremental SBMP algorithms such as RRT \cite{lavalle2001randomized}, on the other hand, avoid this computational overhead by generating random samples to build a connectivity tree online.
	However, the RRT algorithm only guarantees probabilistic completeness. 
	Complementing the exploration module of RRTs with an exploitation module results in asymptotic optimality for these randomized methods i.e., 
	planner converges to the optimal solution almost-surely as the number of samples tend to infinity.
	While the exploration module generates new samples and extends the connectivity tree, the exploitation module processes this tree to improve the current solution.
	The popular RRT* algorithm \cite{karaman2011sampling} locally rewires the tree for exploitation, while RRT$^\#$ \cite{arslan2013use} utilizes dynamic programming to implement a "global rewiring" procedure. RRT$^\#$ ensures optimal connection for each vertex in the current graph at the end of every iteration. 
	Other algorithms such BIT* \cite{gammell2015batch} and FMT* \cite{janson2015fast} also utilize heuristics and dynamic programming to conduct an efficient search and ensure faster convergence compared to RRT*. 
	The DRRT algorithm \cite{hauer2017deformable} uses gradient descent to optimize the location of samples in the tree structure. 
	
	Conventionally, SBMP algorithms have employed a uniform random exploration strategy. This results in an implicit Voronoi bias, leading to a rapid exploration of the search space.
	However, in order to achieve a more focused search, several improvements to the uniform sampling strategy have been suggested.
	These include local biasing \cite{akgun2011sampling}, use of a heuristic-based quality measure \cite{urmson2003approaches}, application of an information-theoretic framework \cite{burns2005toward}, \cite{burns2005single} and translating ideas from A* to the continuous domain \cite{diankov2007randomized}. OBRRT \cite{rodriguez2006obstacle} uses obstacle information to guide the tree growth. MARRT \cite{denny2014marrt} retracts samples onto the medial axis of the free space to obtain high clearance paths. 
	\begin{figure}
		\centering
		\includegraphics[width=.65\columnwidth]{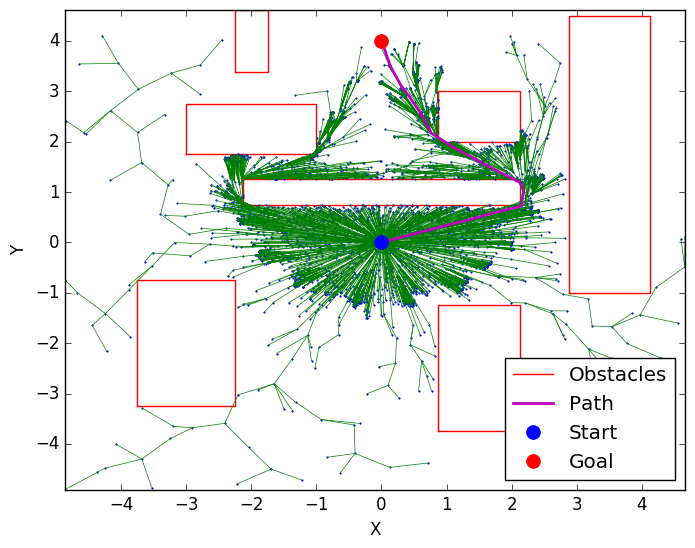}
		\includegraphics[width=.65\columnwidth]{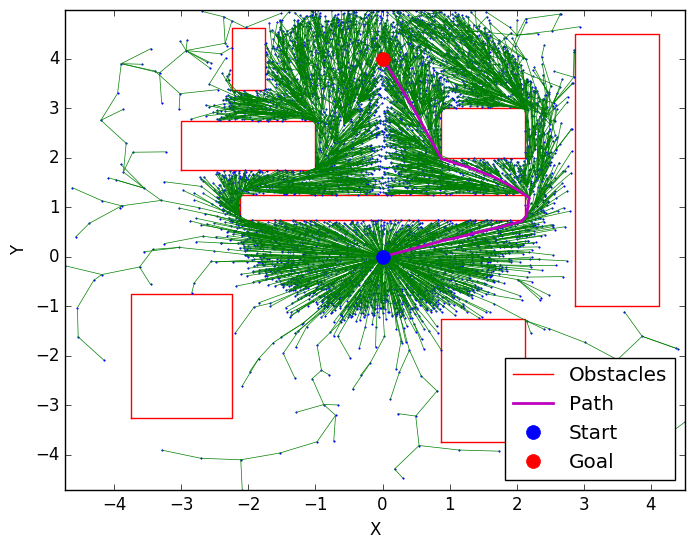}
		\caption{Planning in a multiple obstacle environment with Relevant Region sampling (top) and Informed Sampling (bottom). Note that the Relevant Region focuses on two pertinent homotopy classes whereas the Informed Sampling generates uniform samples inside the ellipsoidal region.}
		\label{fig:mulObstacleEnvironment}
	\end{figure}

	After an initial (sub-optimal) solution is found, exploration can be focused on a subset of the search-space that can potentially further improve the current solution. Doing so can lead to a dramatically faster convergence \cite{gammell2018informed}.
	The Relevant Region, introduced in \cite{arslan2015dynamic}, leverages the current solution cost and the planner's tree structure information to focus search. 
	A selective vertex inclusion procedure \cite{arslan2015dynamic} and a machine learning approach \cite{arslan2015machine} has been proposed to generate new samples in the Relevant Region. 
	However, these approaches fall into the category of rejection sampling methods, which do not scale well for high dimensional problems. 
	The current work rigorously defines the Relevant Region set, analyzes its theoretical properties and presents a \textit{generative} method to sample it.
	This work also extends the Relevant Region framework for planning on general cost-maps. 
	
	Gammell et al. \cite{gammell2018informed} use an admissible heuristic and the current solution cost to define the Informed Set.
	This set includes all points that can potentially improve the current solution.
	Exploration outside the Informed Set is thus redundant. 
	For minimum-length planning problems in Euclidean spaces, the authors of \cite{gammell2018informed} proposed an efficient method to generate samples in the $L_2$-Informed Set.
	The Lebesgue measure of the $L_2$-Informed set decreases as the solution improves, leading to a focused search.
	However, Informed Sampling effectively resorts to uniform random sampling if the Lebesgue measure of the Informed Set is comparable to that of the entire search space. 
	This can happen if the heuristic estimate of the solution cost fails to provide a good enough approximation of the true solution cost.  
	The proposed approach addresses these issues by utilizing cost-to-come information from the planner's tree structure and reducing dependence on heuristics. 
	
	The Expansive space trees (EST) algorithm~\cite{hsu1997path} proceeds by selecting a vertex (with probability inversely proportional to number of vertices in its neighborhood) and generates a new sample in its vicinity.
	Guided ESTs \cite{phillips2004guided} add the A* cost and an exploration term to the vertex weights.
	The SBA* algorithm \cite{persson2014sampling} incorporates a graph density and a constriction measure into the vertex weight. 
	While the algorithm proposed in this work falls into the category of EST-like methods, a crucial difference is that it only expands vertices and generates new samples in the Relevant Region.
	Thus, in contrast to EST and its variants, the proposed algorithm avoids needless exploration.
	
	The SBMP algorithms and the exploration strategies mentioned above are traditionally geared towards finding the (length) optimal path in uniform cost spaces. However, many applications require planning algorithms to find the optimal path with respect to a provided cost function.
	These include the problem of navigation on a rough terrain for a mobile robot (see Fig.\ref{fig:grand_canyon}), safety critical path planning with clearance cost-map (example in Fig.\ref{fig:potential_costmap}), human aware motion planning \cite{mainprice2011planning}, and planning on energy landscapes \cite{jaillet2011randomized}.
	The Transition-based RRT (T-RRT) algorithm~\cite{jaillet2010sampling} takes a user-defined cost function as an additional input and adds a transition test based on the Metropolis criterion to accept or reject potential new states. The transition test favors exploration of low-cost regions of space and leads to better quality paths.
	An enhanced, bi-directional version of T-RRT is presented in~\cite{devaurs2013enhancing}.
	Berenson et al~\cite{berenson2011addressing} combine gradient information within the T-RRT framework to address the issue of navigating cost-space chasms.
	Finally, Devaurs et al~\cite{devaurs2015optimal} combine the filtering properties of the transition test with the local rewiring procedure of RRT* to obtain the asymptotically optimal T-RRT* algorithm.
	While the transition test promotes exploration of low-cost regions, unlike the Informed and Relevant Region sets, it does not focus exploration on to a subset of the search-space based on the current solution.
	Secondly, the probabilistic rejection strategy of the transition test might not scale well to higher dimensional spaces, as the probability of generating a "good" sample that can pass the transition test may decrease rapidly.
	The proposed algorithm addresses these issues by employing a generative sampling approach. It utilizes heuristics, the current solution cost and the cost function information to effectively focus the search in general cost-space settings.
	
	In the following sections, the path planning problem on general cost-maps is formally defined, followed by a comparison between the Informed and the Relevant Region sets. A technique to generate samples in the Relevant Region is then proposed, followed by benchmarking results. 
	\section{PROBLEM DEFINITION}
	\subsection{Path Planning Problem}
	Consider the search space $\mathcal{X}$, which is assumed to be a subset of $\mathbb{R}^d$, where $d$ is a positive integer such that $d \geq 2$. Let $\mathcal{X}_\mathrm{obs}$ denote the obstacle space and  $ \mathcal{X}_\mathrm{free}= \mathrm{cl}(\mathcal{X}\setminus \mathcal{X}_\mathrm{obs}) $ denote the free space.
	Here, $\mathrm{cl}(A)$ represents closure of the set $A \subset \mathbb{R}^d$.  
	Let $\mathcal{M}( A )$ denote the Lebesgue measure of the set $A \subset \mathbb{R}^d$.
	Let $\textbf{x}_\mathrm{s} \in \mathcal{X}_\mathrm{free}$ denote the initial state and $\mathcal{X}_\mathrm{goal} \subset \mathcal{X}_\mathrm{free}$ represent the goal region.
	\begin{figure}[]
		\centering
		\includegraphics[width=0.9\columnwidth,height=.58\columnwidth]{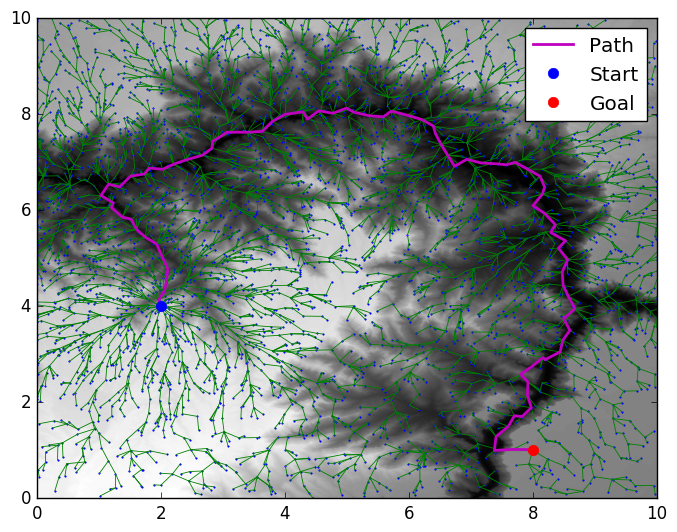}
		\caption{Planning on a terrain cost-map with the proposed sampling strategy. Here, white regions represent rough (high cost) areas and the blacks signify smooth sections. }
		\label{fig:grand_canyon}
	\end{figure}
	Let $C:\mathcal{X} \rightarrow \mathbb{R}_{\geq 0}$ denote a continuous state cost function that imposes a "cost-map" over the search space.
	The path-cost from $\textbf{x}_1 \in \mathcal{X}$ to $\textbf{x}_2 \in \mathcal{X}$ along a path $ \pi:[0,1] \rightarrow \mathcal{X}$ with $\pi(0)=\textbf{x}_1$, $\pi(1)=\textbf{x}_2$ is given by
	\begin{equation}
	\label{eq:ICcost}
	\mathrm{d}_{\pi}(\textbf{x}_1,\textbf{x}_2)= \int_{0}^{1} C(\pi(s)) \  \|\frac{\mathrm{d}\pi(s)}{\mathrm{d}s} \|_2  \ \mathrm{d}s .
	\end{equation}  
	Equation (\ref{eq:ICcost}) represents the integral of cost (IC) along a path as a measure of the path quality. Please see \cite{jaillet2010sampling} for a discussion on different cost criteria.
	If the path is a straight line, i.e., $\pi(s)=\textbf{x}_1+(\textbf{x}_2-\textbf{x}_1)s$, $s \in [0,1]$, then the IC cost is given by
	\begin{equation}
	\label{eq:ICstraightcost}
	\mathrm{d}_\ell(\textbf{x}_1,\textbf{x}_2)= \|\textbf{x}_2-\textbf{x}_1\|_2 \int_{0}^{1} C(\textbf{x}_1+(\textbf{x}_2-\textbf{x}_1)s) \ \mathrm{d}s .
	\end{equation}   
	Note that if $C(\textbf{x})=1$ for all $\textbf{x} \in \mathcal{X}$, then the IC cost in (\ref{eq:ICstraightcost}) reduces to the familiar Euclidean distance $\|\textbf{x}_2-\textbf{x}_1\|_2$.
	Let $\Pi$ denote the set of paths from $\textbf{x}_\mathrm{s}$ to $\mathcal{X}_\mathrm{goal}$.
	Thus, given $(\mathcal{X},\mathcal{X}_\mathrm{obs} ,\textbf{x}_\mathrm{s},\mathcal{X}_\mathrm{goal},C)$, the optimal path planning problem is one of finding the minimum-cost, feasible path $\pi^* \in \Pi$.
	\begin{equation}
	\label{eq:optimalPlanningDef}
	\begin{aligned}
	\arg \min_{\pi \in \Pi}  & \ \mathrm{d}_{\pi}(\textbf{x}_\mathrm{s},\textbf{x}_\mathrm{g}) \\
	\text{subject to:}  & \ \pi(0)= \textbf{x}_\mathrm{s},~ \pi(1)= \textbf{x}_\mathrm{g} \in \mathcal{X}_\mathrm{goal}, \\
	& \ \pi(s) \in \mathcal{X}_\mathrm{free}, ~~~ s \in [0,1].  
	\end{aligned}
	\end{equation}
	Consider the graph $\mathcal{G}=(V,E)$ with a finite set of vertices $V \subset \mathcal{X}_\mathrm{free}$ and edges $E \subseteq V \times V$.
	A spanning tree $\mathcal{T}=(V_s,E_s)$ is embedded in $\mathcal{G}$ such that $V_s=V$ and $E_s=\{ (\textbf{u},\textbf{v}) \in E \ | \ \textbf{v}= \mathsf{parent}(\textbf{u})\}$. 
	Here the function $\mathsf{parent}:V \rightarrow V$ represents the mapping from a vertex to its unique parent vertex.
	By definition, $\mathsf{parent}(\textbf{x}_\mathrm{s})=\textbf{x}_\mathrm{s} $. 
	SBMP algorithms numerically integrate (\ref{eq:ICstraightcost}) to calculate the edge-cost $\mathrm{d}_\ell(\textbf{v},\textbf{u})$ for any edge $(\textbf{u},\textbf{v}) \in E$.
	This work considers cost functions with $C(\textbf{x}) \geq 1$ for all $\textbf{x} \in \mathcal{X}$, so that the edge-cost between any two vertices is at least the Euclidean distance between these two points.
	Given the spanning tree $\mathcal{T}$ in $\mathcal{G}$, the function $\mathrm{g}_\mathcal{T}:V \rightarrow \mathbb{R}_{\geq 0}$ provides the cost-to-come value for any $\textbf{v} \in V$, i.e., it is the sum of the edge-costs along the path from $\textbf{v}$ to $\textbf{x}_\mathrm{s}$ in $\mathcal{T}$.
	A consistent heuristic function on $\mathcal{X}$ (such as the Euclidean distance or $L_2$-norm) is defined as: $\mathrm{h}: \mathcal{X} \times \mathcal{X} \rightarrow \mathbb{R}_{\geq 0} $. The function $\mathrm{h}$ always gives an under-estimate of the path-cost between any two points in the search space, and obeys the triangle inequality.
	The SBMP algorithms solve the planning problem (\ref{eq:optimalPlanningDef}) by drawing random samples from $\mathcal{X}$ and by incorporating the collision-free ones in $\mathcal{G}$.
	An efficient sampling strategy must generate samples so as to find an initial solution or improve the current one.
	The exploration problem in SBMP is one of finding such a sampling strategy to yield faster convergence. 
	%
	%
	\subsection{The Informed Set}
	Let $c_i$ be the cost of the best solution found by the planning algorithm after $i$ iterations. The Informed Set \cite{gammell2018informed} is defined as
	\begin{equation}
	\label{eq:informedSetDef}
	\mathcal{X}_\mathrm{inf} = \{ \textbf{x} \in \mathcal{X} \ | \ \mathrm{h}(\textbf{x}_\mathrm{s},\textbf{x})+\mathrm{h}(\textbf{x},\textbf{x}_\mathrm{g}) < c_i \}.
	\end{equation}    
	Note that $\mathcal{X}_\mathrm{inf}$ uses a heuristic approximation of both the cost-to-come $\mathrm{h}(\textbf{x}_\mathrm{s},\textbf{x})$ and the cost-to-go $\mathrm{h}(\textbf{x},\textbf{x}_\mathrm{g})$ to get an (under)estimate of the \textit{solution cost} constrained to pass through any $\textbf{x} \in \mathcal{X}$.
	Generating new samples in $\mathcal{X}_\mathrm{inf}$ is thus a necessary condition for improving the current solution.
	An algorithm for direct sampling of the $L_2$-Informed Set is given in \cite{gammell2018informed}.   
	\subsection{Relevant Region}
	Consider the set of \textit{relevant vertices} defined as
	\begin{equation}
	\label{eq:relVertices}
	V_\mathrm{rel}= \{ \textbf{v} \in V \ | \ \mathrm{g}_\mathcal{T}(\textbf{v}) +\mathrm{h}(\textbf{v},\textbf{x}_\mathrm{g}) < c_i   \}.
	\end{equation}
	Let $\epsilon > 0$ ball around a relevant vertex $\textbf{v} \in V_\mathrm{rel}$ be defined as
	\begin{equation}
	\label{eq:relEpsilonBall}
	\mathcal{B}^{\epsilon}( \textbf{v})=\{ \textbf{x} \in \mathcal{X} \ | \ \|\textbf{x}- \textbf{v}\|_2 < \epsilon,~ \textbf{v} \in V_\mathrm{rel}   \}.
	\end{equation}
	Consider the estimate of the solution cost constrained to pass through $\textbf{x} \in  \mathcal{B}^{\epsilon}( \textbf{v})$
	\begin{equation}
	\label{eq:relPathcost}
	\hat{f}_\textbf{v}(\textbf{x})= \mathrm{d}_\ell(\textbf{v,x})+\mathrm{g}_\mathcal{T}(\textbf{v})+\mathrm{h}(\textbf{x},\textbf{x}_\mathrm{g}).
	\end{equation}
	The \textit{Relevant Set} around $\textbf{v} \in V_\mathrm{rel}$ is defined as
	\begin{equation}
	\label{eq:relRegion}
	\mathcal{B}^{\epsilon}_\mathrm{rel}(\textbf{v})= \{\textbf{x} \in \mathcal{B}^{\epsilon}( \textbf{v}) \ | \ \hat{f}_\textbf{v}(\textbf{x}) < c_i \}.
	\end{equation}
	Using (\ref{eq:relVertices}), (\ref{eq:relRegion}), the \textit{Relevant Region} is defined as the union of the relevant sets around all relevant vertices
	\begin{equation}
	\label{eq:relRegionSet}
	\mathcal{X}^{\epsilon}_\mathrm{rel} = \bigcup_{\textbf{v} \in V_\mathrm{rel}} \mathcal{B}^{\epsilon}_\mathrm{rel}(\textbf{v}).
	\end{equation}
	Note that, in contrast to $ \mathcal{X}_\mathrm{inf}$ which uses the heuristic estimate $\mathrm{h}(\textbf{x}_\mathrm{s},\textbf{x})$ of the cost-to-come, $\mathcal{X}^{\epsilon}_\mathrm{rel}$  uses $\mathrm{d}_\ell(\textbf{v,x})+\mathrm{g}_\mathcal{T}(\textbf{v})$ from (\ref{eq:relPathcost}).
	This approximation considers the cost-function information (see (\ref{eq:ICstraightcost})), the structure of $\mathcal{T}$, and hence the topology of $\mathcal{X}_\mathrm{free}$.
	While the $L_2$-norm is still a consistent heuristic for cost-maps with $C(\textbf{x}) \geq 1$ for all $\textbf{x} \in \mathcal{X}$, it does not take into account  $C$ or $\mathcal{X}_\mathrm{obs}$.
	It may provide a poor estimate of the solution cost, leading to $\mathcal{M}(\mathcal{X}_\mathrm{inf}) \approx \mathcal{M}(\mathcal{X})$.
	Informed Sampling effectively resorts to uniform random sampling in this case. 
	The set $\mathcal{X}^{\epsilon}_\mathrm{rel}$ alleviates this dependence on a heuristic. 
	The value of $\epsilon$, which controls the size of the Relevant Set, is taken to be slightly greater than the step-size parameter $\eta$ (in our benchmarking simulations, we used $\epsilon=1.5\eta $).
	The step-size parameter $\eta$ in SBMP controls the maximum edge length in $\mathcal{G}$ \cite{gammell2018informed}.
	Note that a very small value of $\epsilon$ would hinder exploration, while a large value of $\epsilon$ may provide a poor estimate of the cost-to-come in (\ref{eq:relPathcost}), as the edge $(\textbf{x},\textbf{v})$ may not be feasible. 
	The following theorem proves that for any $\epsilon>0$, $\mathcal{B}^{\epsilon}_\mathrm{rel}(\textbf{v})$ is not a singleton.    
	\begin{figure}[t]
		\centering
		\includegraphics[width=.8\columnwidth,height=0.48\columnwidth]{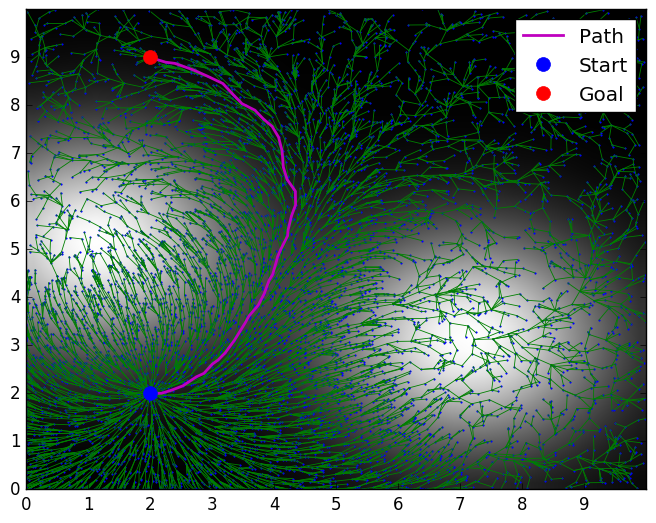}
		\caption{Planning on a "potential-field" like cost-map. The objective is to reach the goal state while avoiding the two danger (white) regions.}
		\label{fig:potential_costmap}
	\end{figure}
	\begin{theorem}  \label{thm:notSingleton}
		For every $\normalfont{\textbf{v}} \in V_\mathrm{rel}$, there exists $\delta>0 $, such that, for all $\normalfont{\textbf{x}} \in \mathcal{B}^{\delta}( \normalfont{\textbf{v}})$, it follows that $\hat{f}_{\normalfont{\textbf{v}}}(\normalfont{\textbf{x}})<c_i $. 
	\end{theorem}
	\begin{proof}
		Note from  (\ref{eq:relPathcost}) that for each given $\normalfont{\textbf{v}} \in V_\mathrm{rel}$ the function
		$\hat{f}_\textbf{v}$ is continuous in  $\normalfont{\textbf{x}}$ since both 
		$\mathrm{d}_\ell( \cdot ,\textbf{v}),\mathrm{h}( \cdot ,\textbf{x}_\mathrm{g}) $ are continuous. 
		Also, $\hat{f}_\textbf{v}(\textbf{v})=\mathrm{g}_\mathcal{T}(\textbf{v}) +\mathrm{h}(\textbf{v},\textbf{x}_\mathrm{g}) < c_i $ since $\textbf{v} \in V_\mathrm{rel} $. 
		Since $\hat{f}_\textbf{v}$ is continuous at $\textbf{v}$, it follows that for any $\zeta > 0$  there exists $\delta >0 $ such that
		$\textbf{x} \in \mathcal{B}^{\delta}( \textbf{v})$
		implies  that $| \hat{f}_\textbf{v}(\textbf{x})-\hat{f}_\textbf{v}(\textbf{v}) |<\zeta $.
		Choosing $\zeta = c_i-\hat{f}_\textbf{v}(\textbf{v}) >0$ one then obtains that  for all
		$\textbf{x} \in \mathcal{B}^{\delta}( \textbf{v})$ we have that 
		$ |\hat{f}_\textbf{v}(\textbf{x})-\hat{f}_\textbf{v}(\textbf{v}) | < c_i-\hat{f}_\textbf{v}(\textbf{v})$ and hence
		$ \hat{f}_\textbf{v}(\textbf{x})<c_i $. 
	\end{proof}
	\begin{corollary}
		Let $\normalfont{\textbf{v}} \in V_\mathrm{rel}$.
		For every $\epsilon > 0$ there exists $\delta>0$ such  $\mathcal{B}^{\delta}( \normalfont{\textbf{v}}) \subset \mathcal{B}^{\epsilon}_\mathrm{rel}(\normalfont{\textbf{v}})$ .
	\end{corollary}
	%
	\begin{theorem}
		\label{thm:relSubsetInf}
		For any $\epsilon>0$, the Relevant Region $\mathcal{X}^{\epsilon}_\mathrm{rel}$ is a subset of the Informed Set $\mathcal{X}_\mathrm{inf}$.  
	\end{theorem}
	\begin{proof}
		Let $\textbf{x} \in \mathcal{X}^{\epsilon}_\mathrm{rel}$. 
		Then there exists $\textbf{v} \in V_\mathrm{rel}$, so that $\textbf{x} \in \mathcal{B}^{\epsilon}_\mathrm{rel}(\textbf{v}) $, 
		and hence $\mathrm{d}_\ell(\textbf{x,v})+\mathrm{g}_\mathcal{T}(\textbf{v})+\mathrm{h}(\textbf{x},\textbf{x}_\mathrm{g}) < c_i$. 
		Since the heuristic function is consistent, 
		$\mathrm{h}(\textbf{x,v})<\mathrm{d}_\ell(\textbf{x,v}) $ and $\mathrm{h}(\textbf{v},\textbf{x}_\mathrm{s}) <\mathrm{g}_\mathcal{T}(\textbf{v})$. 
		Using the triangle inequality, it follows that,
		$\mathrm{h}(\textbf{x}_\mathrm{s},\textbf{x}) < \mathrm{h}(\textbf{x,v}) + \mathrm{h}(\textbf{v},\textbf{x}_\mathrm{s}) $.
		Combining the above inequalities yields
		$\mathrm{h}(\textbf{x}_\mathrm{s},\textbf{x})+\mathrm{h}(\textbf{x},\textbf{x}_\mathrm{g}) < \mathrm{d}_\ell(\textbf{x,v})+\mathrm{g}_\mathcal{T}(\textbf{v})+\mathrm{h}(\textbf{x},\textbf{x}_\mathrm{g}) < c_i $. 
		Hence, $ \textbf{x} \in \mathcal{X}_\mathrm{inf}$. It follows that $\mathcal{X}^{\epsilon}_\mathrm{rel} \subset \mathcal{X}_\mathrm{inf}$. 
	\end{proof} 
	Theorem~\ref{thm:relSubsetInf} implies that generating samples in $\mathcal{X}^{\epsilon}_\mathrm{rel}$ does not lead to redundant exploration outside $ \mathcal{X}_\mathrm{inf}$.
	However, note that sampling in $\mathcal{X}^{\epsilon}_\mathrm{rel}$ is not a necessary condition for improving the current solution, i.e., there may be points $\textbf{x} \in  \mathcal{X}_\mathrm{inf}$ such that $\textbf{x} \notin \mathcal{X}^{\epsilon}_\mathrm{rel}$ which may improve the current solution.
	Relevant Region sampling is thus utilized in conjunction with Informed/Uniform Sampling.
	As shown in the numerical examples later on, this interplay of exploration by Informed Sampling, combined with focusing properties of Relvant Region, leads to accelerated convergence.
	\section{SAMPLING IN THE RELEVANT REGION}
	Since $\mathcal{X}^{\epsilon}_\mathrm{rel} $ depends on $ \mathcal{T}$, a direct sampling strategy is not possible.
	Hence, the proposed sampling strategy proceeds by first selecting a relevant vertex $\textbf{v}_\mathrm{p} \in V_\mathrm{rel} $, sampling a random direction $\hat{\textbf{e}}$, $\|\hat{\textbf{e}}\|_2=1$ and finding the maximum magnitude of travel $\gamma_\mathrm{rel}>0$ along $\hat{\textbf{e}}$, so that for all $\gamma \in (0,\gamma_\mathrm{rel})$ the new sample $\textbf{x}=\textbf{v}_\mathrm{p}+\gamma\hat{\textbf{e}} \in \mathcal{B}^{\epsilon}_\mathrm{rel}(\textbf{v}_\mathrm{p}) $. Please see Fig.~\ref{fig:relregionSchematic}. 
	Note that Theorem \ref{thm:notSingleton} guarantees the existence of $\gamma_\mathrm{rel}$.
	Concretely, the following optimization problem needs to be solved:
	\begin{equation}
	\label{eq:optProblemGammaRel}
	\begin{aligned}
	\sup_{\gamma \in (0,\epsilon)} & \ \gamma, \\
	\text{subject to:} & \ \hat{f}_{\textbf{v}_\mathrm{p}}(\textbf{v}_\mathrm{p}+\gamma\hat{\textbf{e}}) < c_i.
	\end{aligned}
	\end{equation}
	\subsection{Case 1: Uniform Cost-Map} 
	\noindent Consider the problem (\ref{eq:optProblemGammaRel}) with $C(\textbf{x})=1$ for all $\textbf{x} \in \mathcal{X}$. Using the $L_2$-norm heuristic in (\ref{eq:relPathcost}), the inequality in (\ref{eq:optProblemGammaRel}) yields,
	\begin{equation}
	\label{eq:inequalUniCostmap}
	\hat{f}_{\textbf{v}_\mathrm{p}}(\textbf{v}_\mathrm{p}+\gamma\hat{\textbf{e}})=\gamma +\mathrm{g}_\mathcal{T}(\textbf{v}_\mathrm{p})+\|\textbf{v}_\mathrm{p}+\gamma\hat{\textbf{e}}-\textbf{x}_\mathrm{g}\|_2 < c_i.
	\end{equation}  
	Rearrange the terms in (\ref{eq:inequalUniCostmap}) to obtain
	\begin{equation}
	\label{eq:rearrInequalUniCostmap}
	\|\textbf{v}_\mathrm{p}+\gamma\mathrm{\hat{\textbf{e}}}-\textbf{x}_\mathrm{g}\|_2 < c_i-\mathrm{g}_{\mathcal{T}}(\textbf{v}_\mathrm{p})-\gamma.
	\end{equation}
	To ensure that the RHS in (\ref{eq:rearrInequalUniCostmap}) is positive, choose
	\begin{equation}
	\label{eq:uniformFeasibility}
	\gamma  < c_i-\mathrm{g}_\mathcal{T}(\textbf{v}_\mathrm{p}).
	\end{equation}
	Let $\textbf{x}_\mathrm{pg}=\textbf{v}_\mathrm{p}-\textbf{x}_\mathrm{g}$ and $g_\mathrm{gp}=c_i-\mathrm{g}_{\mathcal{T}}(\textbf{v}_\mathrm{p})$.
	Also note that $\textbf{x}_\mathrm{pg}^\mathsf{T}\textbf{x}_\mathrm{pg}=\mathrm{h}^2(\textbf{v}_\mathrm{p},\textbf{x}_\mathrm{g})$  and $\textbf{x}_\mathrm{pg}^\mathsf{T}\hat{\textbf{e}}=\mathrm{h}(\textbf{v}_\mathrm{p},\textbf{x}_\mathrm{g}) \cos\theta$, where $\theta$ is the angle between the vectors $\textbf{x}_\mathrm{pg}$ and $\hat{\textbf{e}}$. Squaring both sides in (\ref{eq:rearrInequalUniCostmap}) yields,  
	\begin{equation*}
	\begin{aligned}
	& \mathrm{h}^2(\textbf{v}_\mathrm{p},\textbf{x}_\mathrm{g})+2\gamma\textbf{x}_\mathrm{pg}^\mathsf{T}\mathrm{\hat{e}}+\gamma^2<g^2_\mathrm{gp}-2\gamma g_\mathrm{gp} + \gamma^2 \\
	& \text{and hence } \gamma <\frac{g^2_\mathrm{gp}-\mathrm{h}^2(\textbf{v}_\mathrm{p},\textbf{x}_\mathrm{g})}{2(\textbf{x}_\mathrm{pg}^\mathsf{T}\mathrm{\hat{\textbf{e}}}+g_\mathrm{gp})}.
	\end{aligned}
	\end{equation*}
	Define the RHS in the above inequality as
	\begin{equation}
	\gamma_\mathrm{uni}= \frac{(c_i-\mathrm{g}_{\mathcal{T}}(\textbf{v}_\mathrm{p}))^2-\mathrm{h}^2(\textbf{v}_\mathrm{p},\textbf{x}_\mathrm{g})}{2\big[\mathrm{h}(\textbf{v}_\mathrm{p},\textbf{x}_\mathrm{g}) \cos\theta+(c_i-\mathrm{g}_{\mathcal{T}}(\textbf{v}_\mathrm{p}))\big]} .
	\end{equation}
	Note that $\gamma_\mathrm{uni}>0$ for $\textbf{v}_\mathrm{p} \in V_\mathrm{rel}$, and attains its maximum value $\overline{\gamma}_\mathrm{uni}$ at $\theta=\pi$, in which case,
	\begin{equation*}
	\overline{\gamma}_\mathrm{uni}=\big( c_i-\mathrm{g}_\mathcal{T}(\textbf{v}_\mathrm{p})+\mathrm{h}(\textbf{v}_\mathrm{p},\textbf{x}_\mathrm{g}) \big) /2,
	\end{equation*}  
	and also, $\overline{\gamma}_\mathrm{uni}< c_i-\mathrm{g}_\mathcal{T}(\textbf{v}_\mathrm{p})$ for $\textbf{v}_\mathrm{p} \in V_\mathrm{rel}$, satisfying (\ref{eq:uniformFeasibility}). 
	Thus, the solution to problem (\ref{eq:optProblemGammaRel}) for uniform cost-map is
	\begin{equation}
	\label{eq:solUniformCostMap}
	\gamma_\mathrm{rel}=\min(\gamma_\mathrm{uni},\epsilon).
	\end{equation} 
	\begin{figure}[]
		\centering
		\includegraphics[width=.55\columnwidth,height=0.3\columnwidth]{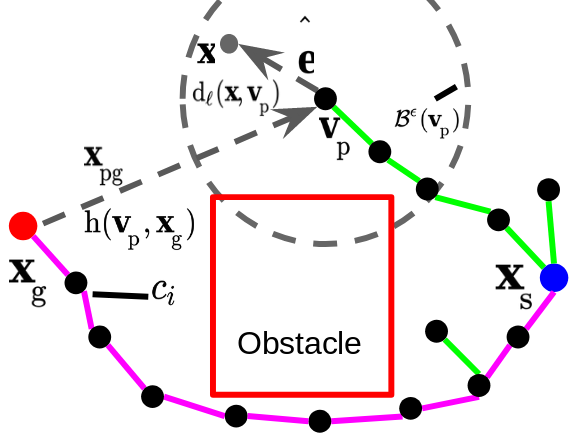}
		\caption{A schematic for Relevant Region sampling.}
		\label{fig:relregionSchematic}
	\end{figure}
	\vspace{-6.00mm}
	\subsection{Case 2: General Cost-Maps}
	\noindent 
	Now consider the problem (\ref{eq:optProblemGammaRel}) with $C(\textbf{x}) > 1$ for all $\textbf{x} \in \mathcal{X}$. The following inequality needs to be solved for $\gamma$,
	\begin{equation}
	\label{eq:inequalGenCostmap}
	\gamma \int_{0}^1 C(\textbf{v}_\mathrm{p}+\gamma\mathrm{\hat{\textbf{e}}}\mathrm{s}) \mathrm{ds}+\mathrm{g}_{\mathcal{T}}(\textbf{v}_\mathrm{p})+\mathrm{h}(\textbf{v}_\mathrm{p}+\gamma\mathrm{\hat{\textbf{e}}},\textbf{x}_\mathrm{g}) <c_i.
	\end{equation}
	Often, $C$ may not have a tractable closed-form expression and hence the planner has access only to the value of $C$ at any point in the search space.
	In order to avoid a computationally expensive procedure to solve (\ref{eq:inequalGenCostmap}), we let  
	\begin{equation} \label{eq:approx}
	\mathrm{d}_\ell(\textbf{v}_\mathrm{p},\textbf{v}_\mathrm{p}+\gamma\mathrm{\hat{\textbf{e}}} )=\gamma \int_{0}^1 C(\textbf{v}_\mathrm{p}+\gamma \mathrm{\hat{\textbf{e}}} \mathrm{s}) \mathrm{ds} \approx \gamma C(\textbf{v}_\mathrm{p})
	\end{equation} 
	Note that (\ref{eq:approx}) uses a zeroth-order approximation of the integrand to estimate the integral. 
	Higher order approximations are possible, but these will result in a computationally more involved process to find $\gamma$ (see below).
	It follows from (\ref{eq:approx}) that 
	\begin{equation} \label{eq:approxInequalGenCostmap}
	\gamma C(\textbf{v}_\mathrm{p}) +\mathrm{g}_{\mathcal{T}}(\textbf{v}_\mathrm{p})+\|\textbf{v}_\mathrm{p}+\gamma\hat{\textbf{e}}-\textbf{x}_\mathrm{g}\|_2 < c_i,
	\end{equation}
	\begin{equation} \label{eq:approxInequalGenCostmapB}
	\text{or, } \|\textbf{v}_\mathrm{p}+\gamma\hat{\textbf{e}}-\textbf{x}_\mathrm{g}\|_2 < c_i-\mathrm{g}_{\mathcal{T}}(\textbf{v}_\mathrm{p})-\gamma C(\textbf{v}_\mathrm{p}).
	\end{equation}
	To ensure that the RHS of (\ref{eq:approxInequalGenCostmapB}) is positive, choose
	\begin{equation}
	\label{eq:generalFeasibility}
	\gamma < \big(c_i-\mathrm{g}_{\mathcal{T}}(\textbf{v}_\mathrm{p})\big)/C(\textbf{v}_\mathrm{p}).
	\end{equation}
	Let again $\textbf{x}_\mathrm{pg}=\textbf{v}_\mathrm{p}-\textbf{x}_\mathrm{g}$, $g_\mathrm{gp}=c_i-\mathrm{g}_{\mathcal{T}}(\textbf{v}_\mathrm{p})$, $\textbf{x}_\mathrm{pg}^\mathsf{T}\textbf{x}_\mathrm{pg}=\mathrm{h}^2(\textbf{v}_\mathrm{p},\textbf{x}_\mathrm{g})$ and $\textbf{x}_\mathrm{pg}^\mathsf{T}\hat{\textbf{e}}=\mathrm{h}(\textbf{v}_\mathrm{p},\textbf{x}_\mathrm{g}) \cos\theta$, where $\theta$ is the angle between the vectors $\textbf{x}_\mathrm{pg}$ and $\hat{\textbf{e}}$.
	Squaring both sides in (\ref{eq:approxInequalGenCostmapB}) and simplifying yields,
	\begin{equation}
	\label{eq:quadraticInequality}
	\gamma^2(C^2(\textbf{v}_\mathrm{p})-1)-2\gamma ( g_\mathrm{gp}C(\textbf{v}_\mathrm{p})+\textbf{x}_\mathrm{pg}^\mathsf{T}\hat{\textbf{e}} )
	+ g^2_\mathrm{gp}-\mathrm{h}^2(\textbf{v}_\mathrm{p},\textbf{x}_\mathrm{g})>0.
	\end{equation}
	Let $\gamma_1, \gamma_2$ be the roots of the quadratic equation corresponding to inequality (\ref{eq:quadraticInequality}), and assume $\gamma_2>\gamma_1$.
	\begin{equation}
	\label{eq:solQuadraticInequal}
	\begin{aligned}
	& \gamma_2= \frac{g_\mathrm{gp}C(\textbf{v}_\mathrm{p})+\mathrm{h}(\textbf{v}_\mathrm{p},\textbf{x}_\mathrm{g}) \cos\theta  + \sqrt{\Delta}}{(C^2(\textbf{v}_\mathrm{p})-1)} \\
	& \gamma_1 = \frac{ g_\mathrm{gp}C(\textbf{v}_\mathrm{p})+\mathrm{h}(\textbf{v}_\mathrm{p},\textbf{x}_\mathrm{g}) \cos\theta  - \sqrt{\Delta}}{(C^2(\textbf{v}_\mathrm{p})-1)} \\
	& \Delta = (g_\mathrm{gp}C(\textbf{v}_\mathrm{p})+\mathrm{h}(\textbf{v}_\mathrm{p},\textbf{x}_\mathrm{g}) \cos\theta)^2\\
	& \qquad \qquad -(C^2(\textbf{v}_\mathrm{p})-1)(g^2_\mathrm{gp}-\mathrm{h}^2(\textbf{v}_\mathrm{p},\textbf{x}_\mathrm{g})).
	\end{aligned}
	\end{equation}
	The maximum and minimum values of the radicand $\Delta$ are obtained at $\theta=0$ and $\theta=\pi$, respectively, where
	\begin{equation}
	(g_\mathrm{gp}-\mathrm{h}(\textbf{v}_\mathrm{p},\textbf{x}_\mathrm{g})C(\textbf{v}_\mathrm{p}))^2 \leq \Delta \leq (g_\mathrm{gp}+\mathrm{h}(\textbf{v}_\mathrm{p},\textbf{x}_\mathrm{g})C(\textbf{v}_\mathrm{p}))^2.
	\end{equation}    
	Hence, $\gamma_1, \gamma_2 \in \mathbb{R}_{\geq 0}$ for $\textbf{v}_\mathrm{p} \in V_\mathrm{rel}$. Then (\ref{eq:quadraticInequality}) yields,
	\begin{equation}
	\begin{aligned}
	(\gamma-\gamma_1)(\gamma-\gamma_2)>0.\ 
	\text{equivalently}, \gamma >\gamma_2 \ \text{or} \ \gamma < \gamma_1.
	\end{aligned}
	\end{equation}
	Consider the larger root $\gamma_2$ from (\ref{eq:solQuadraticInequal}).
	The minimum value of $\gamma_2$ is attained when $\theta=\pi$, so that 
	\begin{equation}
	\label{eq:gammaBar2Ineqaual}
	\begin{aligned}
	\gamma_2 \geq \frac{g_\mathrm{gp}C(\textbf{v}_\mathrm{p})-\mathrm{h}(\textbf{v}_\mathrm{p},\textbf{x}_\mathrm{g})+|g_\mathrm{gp}-\mathrm{h}(\textbf{v}_\mathrm{p},\textbf{x}_\mathrm{g})C(\textbf{v}_\mathrm{p})|}{(C^2(\textbf{v}_\mathrm{p})-1)}.
	\end{aligned}
	\end{equation}
	\IncMargin{.5em}
	\begin{algorithm}[t]
		\caption{Sampling Algorithm}
		\label{alg:algFlow}	
		$V \leftarrow \{ \textbf{x}_\mathrm{s} \} $; $E \leftarrow \phi$; $\mathcal{G} \leftarrow (V,E)$\;	
		\For{$i=1:N$} 
		{
			$c_i \leftarrow \min_{\textbf{v} \in V_\mathrm{goal}}  \mathrm{g}_\mathcal{T}(\textbf{v})$\;
			$u_\mathrm{rand} \sim \mathcal{U}(0,1)$\;
			\eIf{$u_\mathrm{rand}<p_\mathrm{rel} \ \normalfont \text{and} \ c_i<\infty$}
			{
				$\textbf{v}_\mathrm{p} \leftarrow \mathsf{chooseVertex}(V_\mathrm{rel}$)\;
				$\hat{\textbf{e}}\leftarrow \mathsf{generateDirection}()$\;
				$\gamma_\mathrm{rel} \leftarrow \mathsf{RelevantStepLimit}(\textbf{v}_\mathrm{p},\hat{\textbf{e}})$\;
				$u_\mathrm{rand} \sim \mathcal{U}(0,1)$\;
				$\textbf{x}_\mathrm{rand} \leftarrow  \textbf{v}_\mathrm{p}+(u_\mathrm{rand})^\frac{1}{d}  \gamma_\mathrm{rel}\hat{\textbf{e}}$\;
			}
			{
				$\textbf{x}_\mathrm{rand} \leftarrow \mathsf{InformedSampling()}$
			}
			$\textbf{x}_{\mathrm{new}} \leftarrow \mathsf{Extend}(\textbf{x}_\mathrm{rand})$\;		
			$\mathsf{Exploitation}(\mathcal{G})$\;
		}
		\KwRet $\mathcal{G}$
	\end{algorithm}
	\DecMargin{.5em}
	Define the RHS in (\ref{eq:gammaBar2Ineqaual}) as $\overline{\gamma}_{2}$. Simplifying yields,
	\begin{equation}
	\overline{\gamma}_{2} =
	\Bigg\{ \begin{array}{cc}
	\frac{g_\mathrm{gp}+\mathrm{h}(\textbf{v}_\mathrm{p},\textbf{x}_\mathrm{g})}{C(\textbf{v}_\mathrm{p})+1},  &  g_\mathrm{gp}<\mathrm{h}(\textbf{v}_\mathrm{p},\textbf{x}_\mathrm{g})C(\textbf{v}_\mathrm{p}),\\
	\frac{g_\mathrm{gp}-\mathrm{h}(\textbf{v}_\mathrm{p},\textbf{x}_\mathrm{g})}{C(\textbf{v}_\mathrm{p})-1},  & g_\mathrm{gp}>\mathrm{h}(\textbf{v}_\mathrm{p},\textbf{x}_\mathrm{g})C(\textbf{v}_\mathrm{p}).
	\end{array}
	\end{equation}
	Note that $\overline{\gamma}_{2}>g_\mathrm{gp}/C(\textbf{v}_\mathrm{p})$. This implies $\gamma_2>g_\mathrm{gp}/C(\textbf{v}_\mathrm{p})$, violating (\ref{eq:generalFeasibility}). Thus, $\gamma>\gamma_2$ is an infeasible solution of (\ref{eq:approxInequalGenCostmap}).
	Next, consider $\gamma_1$.
	Differentiating with respect to $\theta$, the extrema are obtained at $\theta=0,\pi$.
	Calculating the second derivative yields, $\gamma''_1(\theta=0)>0$ and $\gamma''_1(\theta=\pi)<0$.
	The maximum value of $\gamma_1$ obtained at $\theta=\pi$ is given by 
	\begin{equation}
	\begin{aligned}
	\overline{\gamma}_{1} &=
	\Bigg\{ \begin{array}{cc}
	\frac{g_\mathrm{gp}+\mathrm{h}(\textbf{v}_\mathrm{p},\textbf{x}_\mathrm{g})}{C(\textbf{v}_\mathrm{p})+1},  &  g_\mathrm{gp}>\mathrm{h}(\textbf{v}_\mathrm{p},\textbf{x}_\mathrm{g})C(\textbf{v}_\mathrm{p}),\\
	\frac{g_\mathrm{gp}-\mathrm{h}(\textbf{v}_\mathrm{p},\textbf{x}_\mathrm{g})}{C(\textbf{v}_\mathrm{p})-1},  & g_\mathrm{gp}<\mathrm{h}(\textbf{v}_\mathrm{p},\textbf{x}_\mathrm{g})C(\textbf{v}_\mathrm{p}) .
	\end{array}
	\end{aligned}
	\end{equation}
	Now, $\overline{\gamma}_{1} <g_\mathrm{gp}/C(\textbf{v}_\mathrm{p})$. This implies $\gamma_1<g_\mathrm{gp}/C(\textbf{v}_\mathrm{p})$. It follows that $\gamma<\gamma_1$ satisfies (\ref{eq:generalFeasibility}). Thus, the solution to problem (\ref{eq:optProblemGammaRel}) with the approximation in (\ref{eq:approxInequalGenCostmap}) is
	\begin{equation}
	\label{eq:solGenCostMap}
	\gamma_\mathrm{rel}=\min(\gamma_{1},\epsilon).
	\end{equation}
	For the special case when $\Delta=0$ and $\gamma_1=\gamma_2=\gamma_\mathrm{c}$, inequality (\ref{eq:quadraticInequality}) simplifies to $(\gamma-\gamma_\mathrm{c})^2>0$.
	Considering (\ref{eq:generalFeasibility}) yields $\gamma_\mathrm{rel}=\min(g_\mathrm{gp}/C(\textbf{v}_\mathrm{p}),\epsilon) $.
	Note that if $C(\textbf{v}_\mathrm{p})=1$, then inequality (\ref{eq:inequalGenCostmap}) reduces to (\ref{eq:inequalUniCostmap}) and the analysis for uniform cost-maps is applicable.
	%
	\section{PROPOSED ALGORITHM}
	The outline of the proposed algorithm in given in Algorithm \ref{alg:algFlow}. The procedure initializes a vertex at the start state $\textbf{x}_\mathrm{s}$.
	At every iteration, the current best solution cost $c_i$ is updated (line 3).
	If a sub-optimal solution exists ($c_i$ is finite), with probability $p_\mathrm{rel}$ (line 5), Relevant Region sampling is employed to generate a new random sample $\textbf{x}_\mathrm{rand}$.
	Otherwise, conventional Informed Sampling is used.
	Relevant Region sampling consists of first choosing a relevant vertex $\textbf{v}_\mathrm{p}$, generating a random direction $\hat{\textbf{e}}$ and calculating the maximum magnitude of travel along $\hat{\textbf{e}}$ (line 6-8).
	If $C(\textbf{v}_\mathrm{p})=1$, then (\ref{eq:solUniformCostMap}) is used for obtaining $\gamma_\mathrm{rel}$ along $\hat{\textbf{e}}$, else (\ref{eq:solGenCostMap}) is used.
	The exponent $1/d$ (line 10) biases the travel magnitude towards $\gamma_\mathrm{rel}$ and promotes exploration.
	After $\textbf{x}_\mathrm{rand}$ is generated, conventional SBMP modules incorporate a new vertex $\textbf{x}_\mathrm{new}$ in $\mathcal{G}$ (line 13).
	These include: a) finding the nearest neighbor $\textbf{x}_\mathrm{nearest}$ to $\textbf{x}_\mathrm{rand}$ in $\mathcal{G}$; b) local steering from $\textbf{x}_\mathrm{nearest}$ in the direction of $\textbf{x}_\mathrm{rand}$ to obtain $\textbf{x}_\mathrm{new}$; c) ensuring feasibility of edge-connections in the neighborhood of $\textbf{x}_\mathrm{new}$.
	This is followed by the exploitation module (local/global rewiring, etc).
	\begin{figure}[]
		\centering
		\includegraphics[height=.43\columnwidth]{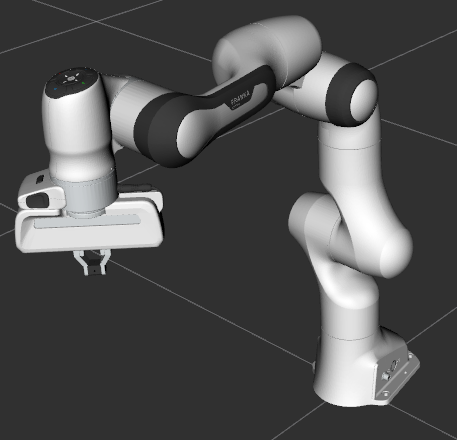}
		\includegraphics[height=.43\columnwidth]{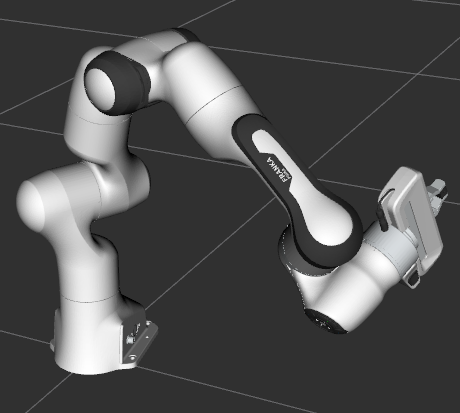}
		\caption{Planning for 7 DOF Panda Arm in the joint space from the start state (left) to a given joint goal state (right).}
		\label{fig:panda}
		\includegraphics[width=.8\columnwidth,height=0.5\columnwidth]{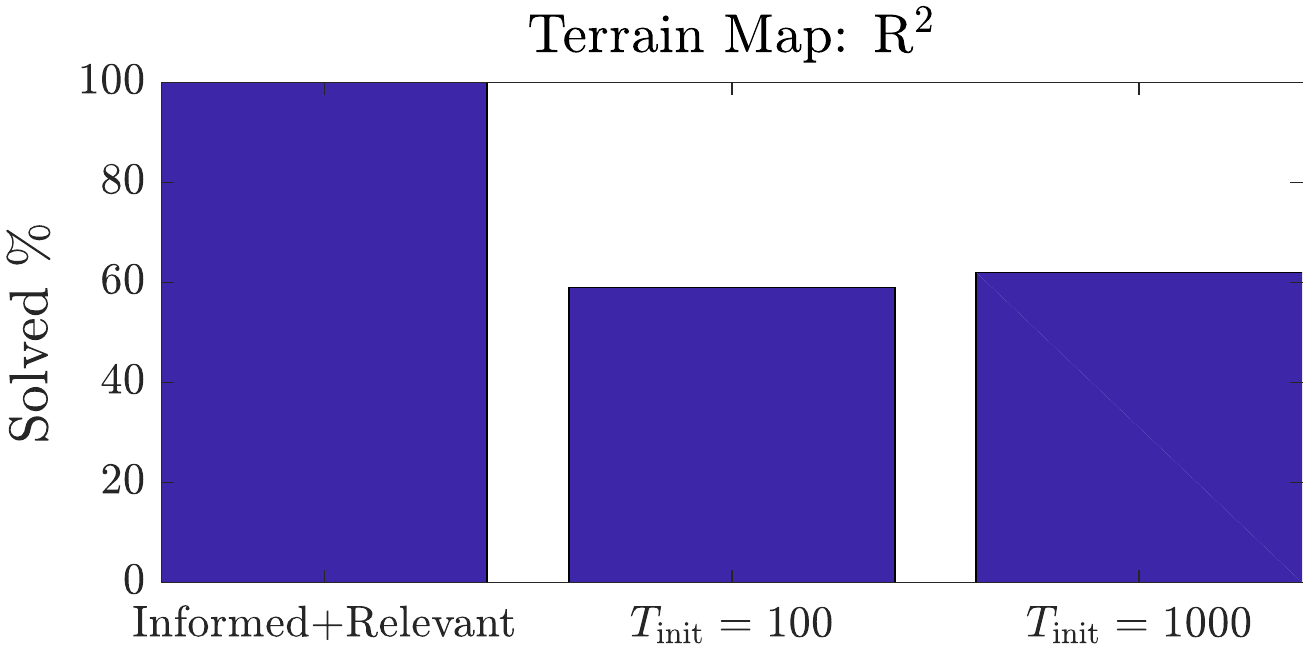}
		\caption{Percentage of successful trials (where planner found a feasible solution) with different sampling strategies.}
		\label{fig:solvedPlotTerrain2D}
	\end{figure}
	The $\mathsf{chooseVertex}$ module selects a relevant vertex to be expanded from the set $V_\mathrm{rel}$. Similar to the procedure in Guided-ESTs \cite{phillips2004guided} a weight $q_\textbf{v}$ is allocated for each $\textbf{v} \in V_\mathrm{rel} $.
	\begin{equation}
	\label{eq:weightQRel}
	q_\textbf{v} =\lambda_1 p_\textbf{v} + \lambda_2 d_\textbf{v} + \lambda_3 \big(\mathrm{g}_{\mathcal{T}}(\textbf{v})+\mathrm{h}(\textbf{v},\textbf{x}_\mathrm{g})\big)/c_i.
	\end{equation} 
	Here, $p_\textbf{v}$ represents the number of times $\textbf{v}$ has been selected in the past. This penalizes multiple selections and the exploration of the region around a particular vertex.
	The second term, $d_\textbf{v}$ is the number of edges connected to $\textbf{v}$. It promotes sampling in relatively unexplored regions. 
	The last term $0<\big(\mathrm{g}_{\mathcal{T}}(\textbf{v})+\mathrm{h}(\textbf{v},\textbf{x}_\mathrm{g})\big)/c_i <1$ is the estimate of the solution cost through $\textbf{v}$, normalized by the current best cost.
	\setlength{\belowcaptionskip}{-15pt}
	\begin{figure*}[]
		\centering
		\includegraphics[width=.65\columnwidth,height=0.32\columnwidth]{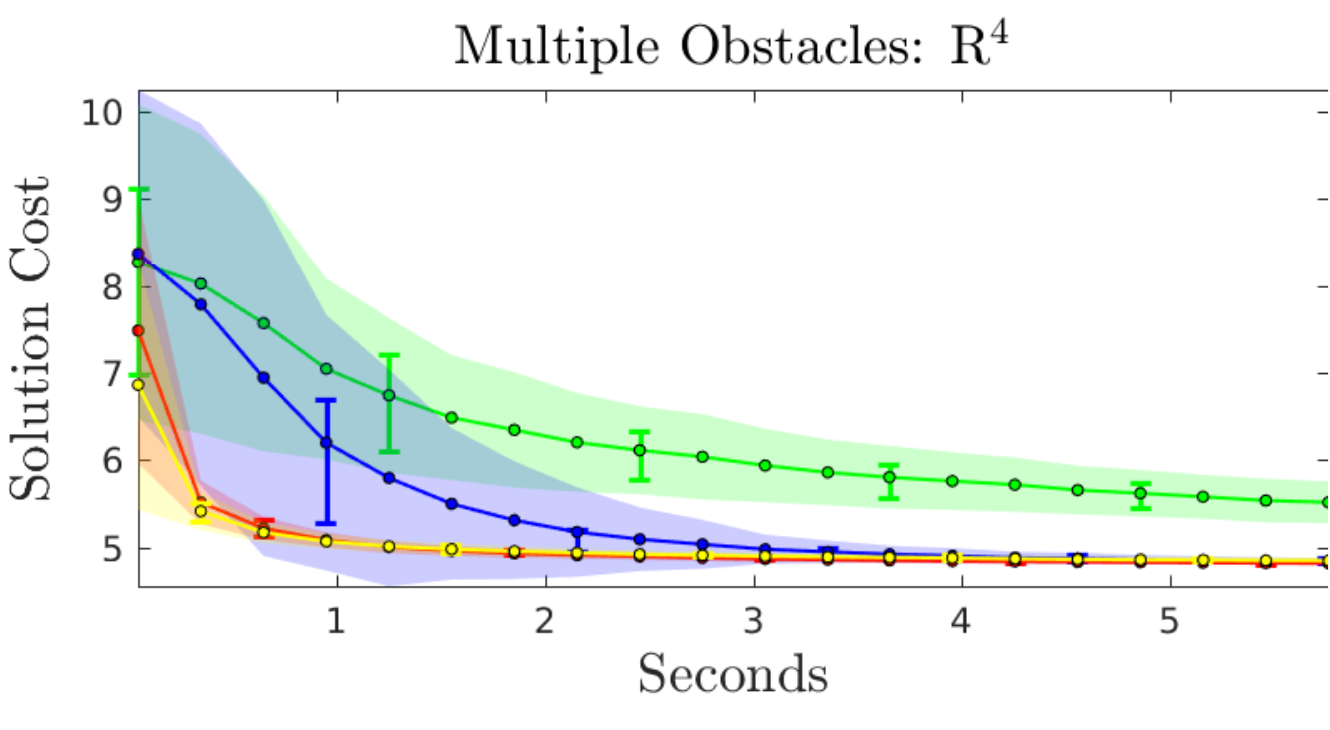}
		\includegraphics[width=.65\columnwidth,height=0.32\columnwidth]{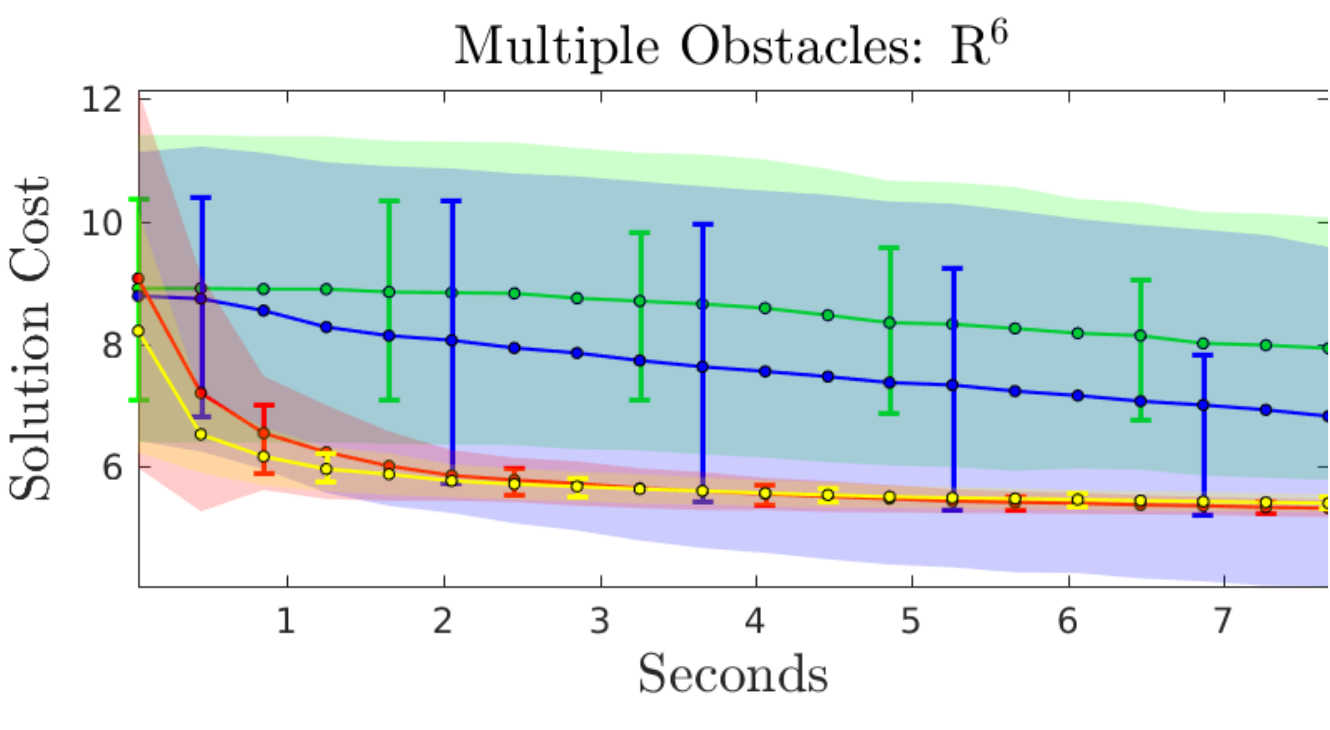}
		\includegraphics[width=.65\columnwidth,height=0.32\columnwidth]{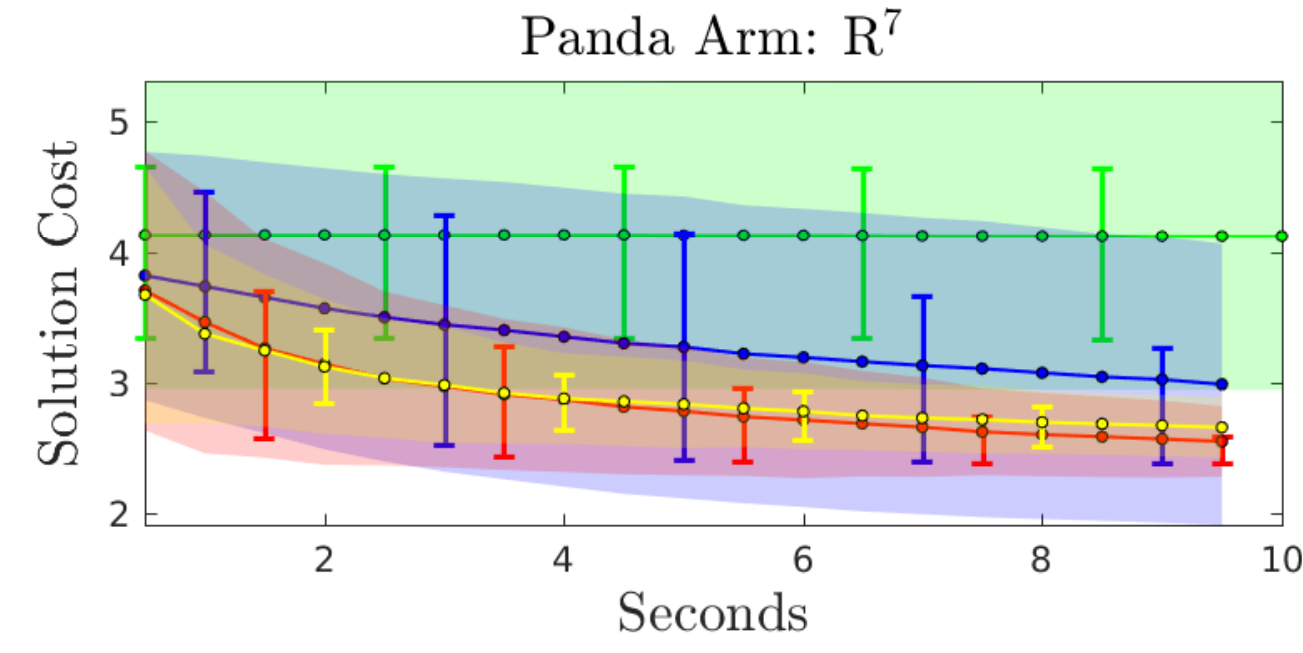}
		
		\includegraphics[width=.65\columnwidth,height=0.32\columnwidth]{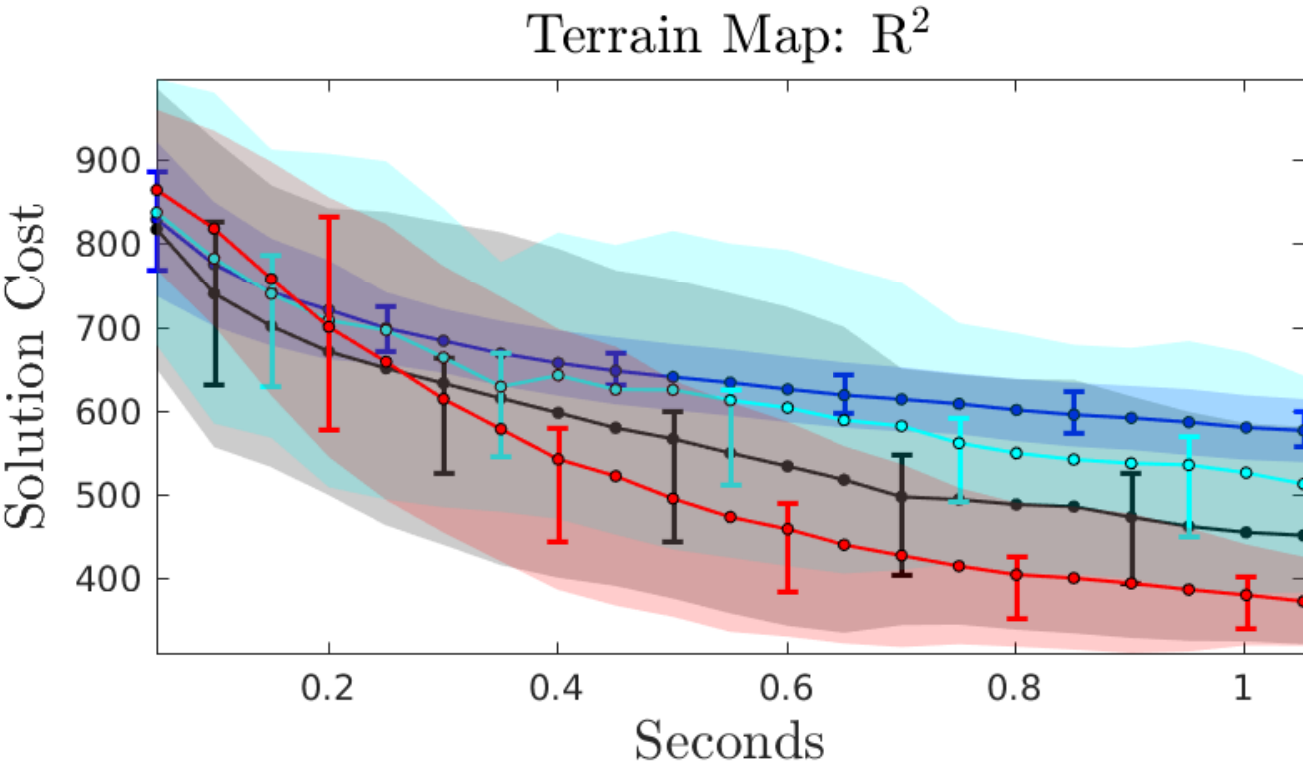}
		\includegraphics[width=.65\columnwidth,height=0.32\columnwidth]{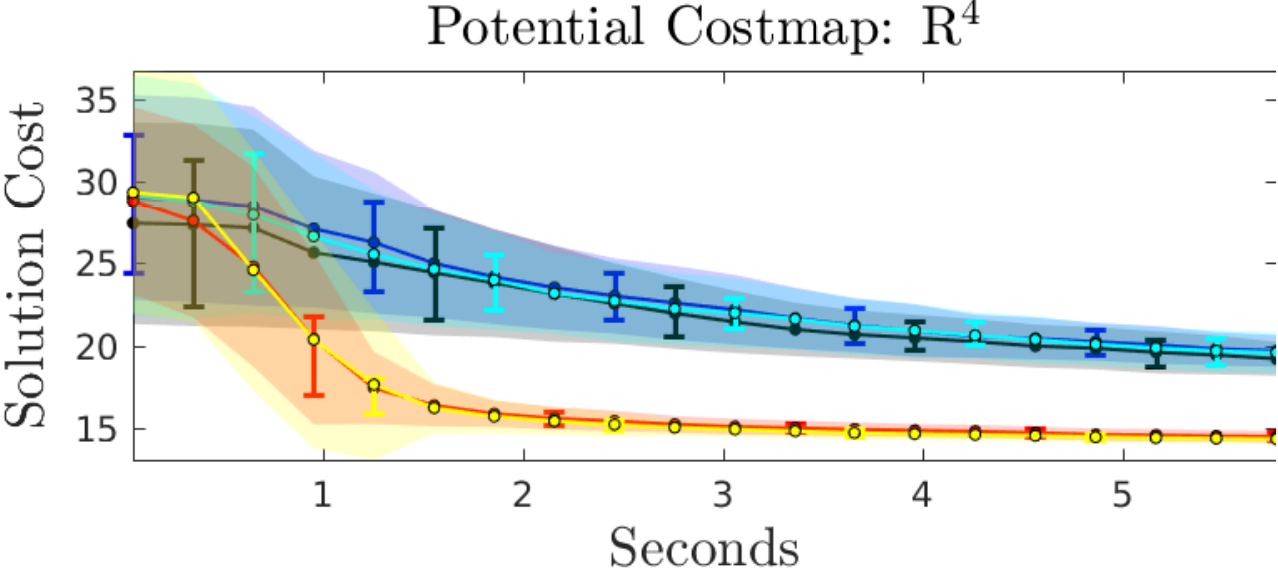}
		\includegraphics[width=.65\columnwidth,height=0.32\columnwidth]{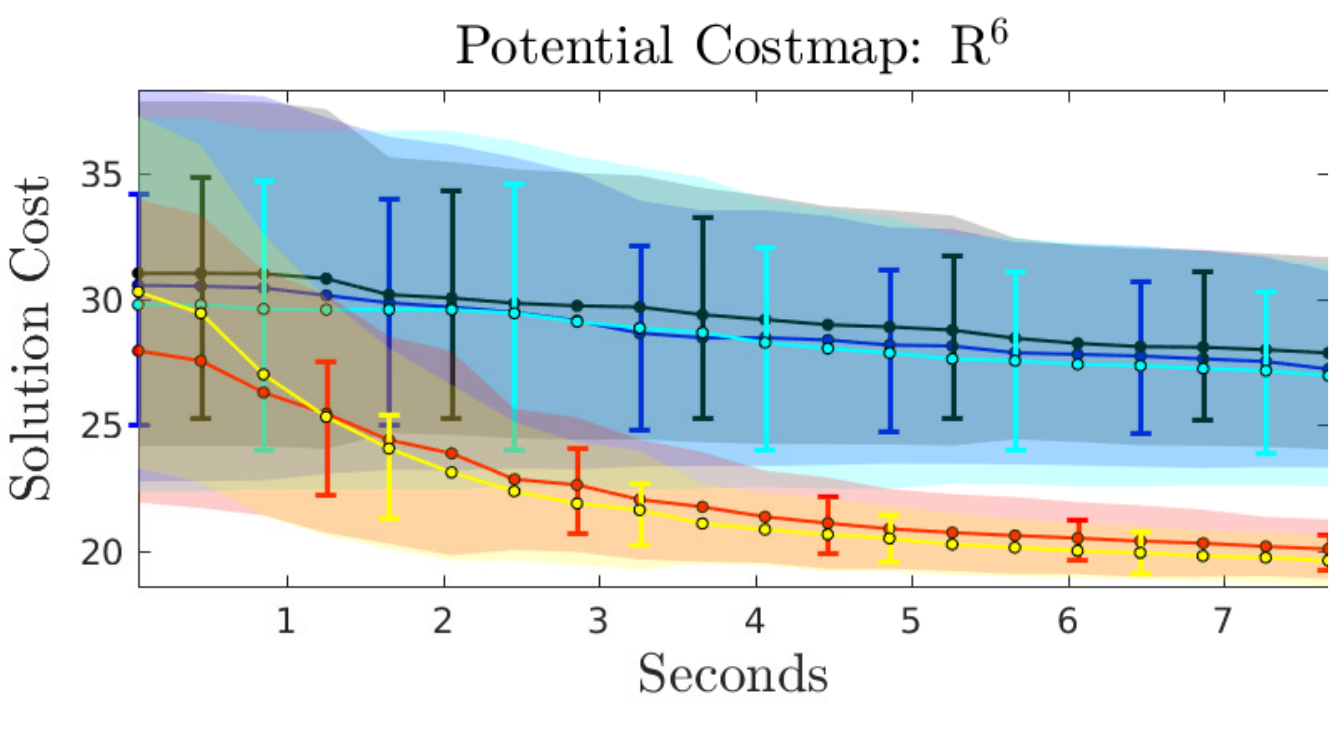}	
		
		\centering
		\includegraphics[width=1.8\columnwidth,height=0.07\columnwidth]{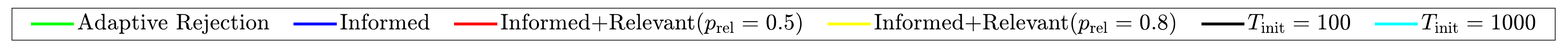}
		\caption{Convergence plots for different sampling methods in various test environments. Solid lines indicate the average value and the standard deviation is shaded. Error bar indicate the upper and lower quartiles.}
		\label{fig:convegencePlot}
	\end{figure*}
	This prioritizes exploration of regions with low solution cost estimates.
	The parameters $(\lambda_1,\lambda_2,\lambda_3) >0$ modulate the behavior of the selection algorithm. 
	A large value of $\lambda_3$ leads to a greedy focus on low solution cost areas, whereas increasing $\lambda_1,\lambda_2$ promotes exploration. 
	A binary heap is used to update and sort $V_\mathrm{rel}$ according to the weight in (\ref{eq:weightQRel}).
	A relevant vertex $\textbf{v}_\mathrm{p}$ is selected by choosing randomly from the top $n_q$ elements in the sorted list. This injects randomness in the selection process and promotes desirable exploration.
	\section{NUMERICAL EXPERIMENTS}
	The performance of the proposed sampling method was benchmarked against direct Informed Sampling~\cite{gammell2018informed} and the third variant of adaptive rejection sampling (described in~\cite{arslan2015dynamic}) in uniform cost-space environments (length-optimal planning).
	For all experiments, the exploration strategies were paired with RRT$^{\#}$'s dynamic programming based global rewiring for exploitation.
	In general cost-map environments, benchmarking was done against Informed Sampling and T-RRT$^{\#}$ (combining conventional RRT$^{\#}$ with the transition-test described in \cite{jaillet2010sampling}) with different initial temperatures $T_\mathrm{init}$.
	All the algorithms were implemented in C++ using the popular OMPL framework~\cite{sucan2012open}, 
	and the tests were run using OMPL's standardized benchmarking tools \cite{moll2015benchmarking}.
	Please see, \url{https://github.gatech.edu/DCSL/relevant_region}.
	A 64-bit desktop PC with 64 GB RAM and an Intel Xeon(R) Processor running Ubuntu 16.04 OS was used.
	The data was recorded over 100 trials for all the cases.
	The proposed algorithm used the following parameter values:$\epsilon=1.5\eta $, $(\lambda_1,\lambda_2,\lambda_3)=(10,5,100),n_q=10$.
	A goal bias of 5\% was used in all sampling methods.
	A description of the different environments is provided below.
	%
	\subsection{Uniform Cost-Map Cases}
	\textit{Multiple Obstacle World:} This environment is illustrated in Fig.~\ref{fig:mulObstacleEnvironment}.
	The 2D environment was extended to $\mathbb{R}^4$ and $\mathbb{R}^6$ by imparting a length of 2 units symmetrically to all of the obstacles.
	A step-size of $\eta=0.6$ and $\eta=1.2$ was used in $\mathbb{R}^4$ and $\mathbb{R}^6$ respectively.
	
	\textit{Panda Arm:} A planning problem for Panda Arm (by Franka Enmika) is illustrated in Fig.~\ref{fig:panda}. The objective was to find a minimum length path in a 7-dimensional configuration (joint) space with joint limits ($\mathbb{R}^7$). These limits and collision checking module were implemented using MoveIt! \cite{chitta2012moveit}. The step-size was set to $\eta=0.7$ for this example.
	\subsection{General Cost-Map Cases}      
	\textit{Terrain Map:} A 2D terrain map shown in Fig.~\ref{fig:grand_canyon} consists of rough, high-cost white areas and the easily navigable black regions.
	The step-size was set to $\eta=0.3$ for this example.
	
	\textit{Potential Cost-Map:} The environment in Fig.~\ref{fig:potential_costmap} emulates the problem of finding the shortest path while staying away from danger areas (white regions). The cost function is defined as
	\begin{equation}
	C(\textbf{x})=1 + 9\big( e^{ -\frac{\|\textbf{x}_1^d -\textbf{x} \|_2^2 }{5} } + e^{ -\frac{\|\textbf{x}_2^d -\textbf{x} \|_2^2 }{5} } \big).
	\end{equation} 
	Here, $\textbf{x}_1^d$, $\textbf{x}_2^d$ are the center points of the danger regions.
	A step-size of $\eta=0.6$ and $\eta=1.5$ was used in $\mathbb{R}^4$ and $\mathbb{R}^6$ version of the environment respectively.
	\vspace{-1.00mm}
	\section{CONCLUSION}
	This work proposes a novel algorithm to sample the Relevant Region set, a subset of the Informed Set, for SBMP.
	The Relevant Region set considers the topology of $\mathcal{X}_\mathrm{free}$, reduces the dependence on heuristics, and effectively focuses the search to accelerate convergence.
	Numerical experiments validate the utility of Relevant Region sampling in conjunction with Informed/Uniform Sampling.
	The proposed method leads to faster convergence in all cases (see Fig.~\ref{fig:convegencePlot}). This is observed especially in higher dimensional problem instances.
	Transition-test based exploration is more effective than purely Uniform/Informed Sampling for planning on general cost-maps. However, the tendency to (probabilistically) reject samples may hinder exploration in some cases.  
	This can be seen in the terrain cost-map (Fig.~\ref{fig:grand_canyon}) which is similar to the cost-space chasms scenario described in \cite{berenson2011addressing}.
	As conveyed in Fig.~\ref{fig:solvedPlotTerrain2D}, the transition-test based exploration fails to find a feasible solution in roughly 40\% of total trials, whereas the proposed method finds a solution in all trials and also accelerates the convergence.
	
	This work presents many avenues for future research.
	The proposed method inherits an additional computational overhead to maintain $V_\mathrm{rel}$. However, this can be alleviated by leveraging ideas from sparse tree planners \cite{dobson2014sparse}, \cite{li2015sparse} to maintain a sparse set of vertices in $V_\mathrm{rel}$.  
	The cost function's gradient information (if available) can be used to bias the search.
	Data from past iterations can also be used to infer the nature of cost-map for intelligent exploration.

	\bibliographystyle{IEEEtran}
	\bibliography{refs}
\end{document}